\documentclass{article} 
\usepackage{nips14submit_e,times}

\pdfoutput=1

\usepackage[utf8]{inputenc}
\usepackage[T1]{fontenc}

\usepackage[ngerman,english]{babel}
\usepackage{csquotes}

\usepackage[
            backend=bibtex,
            natbib=true,
            style=numeric,
            sorting=none,
            minnames=1,
            firstinits=true]{biblatex}

\usepackage{amsmath}
\usepackage{amssymb}
\usepackage{mathtools}
\usepackage{tensor}
\usepackage{MnSymbol}
\usepackage{colonequals}

\usepackage{graphicx}
\usepackage{array} 
\usepackage{booktabs} 
\usepackage{multirow}
\usepackage{caption}
\usepackage{subcaption}
\usepackage{enumitem} 

\usepackage{tikz}
\usetikzlibrary{arrows}
\usepackage{pgfplots}  
\pgfplotsset{compat=newest}
\pgfplotsset{plot coordinates/math parser=false}

\newlength\figureheight
\newlength\figurewidth

\usepackage[ruled,vlined]{algorithm2e}

\definecolor{MPG}{RGB}{000,125,122}

\usepackage{tgheros}

\usepackage{slantsc}
\usepackage{microtype}

\usepackage[setpagesize=false,
            colorlinks=true,
            linkcolor=black,
            citecolor=black,
            urlcolor=black,
            hyperfootnotes=false]{hyperref}

\usepackage{pdflscape}

\usepackage{slantsc}
\usepackage{microtype}

\newcommand{\matlab}{\texttt{MATLAB}}

\renewcommand{\Re}{\mathbb{R}} 

\newcommand{\N}{\mathcal{N}} 
\newcommand{\GP}{\mathcal{GP}} 

\renewcommand{\vec}{\boldsymbol} 
\newcommand{\mat}{\boldsymbol} 
\newcommand{\Id}{\vec{I}} 

\newcommand{\g}{\,|\,} 
\newcommand{\de}{\partial} 
\newcommand{\ec}{\equalscolon} 

\newcommand{\abs}[1]{\lvert #1\rvert} 
\newcommand{\norm}[1]{\lvert\lvert #1\rvert\rvert}
\newcommand{\ind}[1]{\mathbb{I}_{#1}} 

\newcommand{\q}{\quad}
\newcommand{\qq}{\qquad}

\newcommand{\qqqq}{\qquad\qquad}

\usepackage{hyperref}
\usepackage{url}

\usepackage{nicefrac}

\usepackage{amsthm}
\newtheorem{thm}{Theorem}

\newcommand{\dk}{\tensor[^{\de}]{k}{}}
\newcommand{\kd}{\tensor{k}{^{\de}}}
\newcommand{\dkd}{\tensor[^{\de}]{k}{^{\de}}}

\newcommand{\dKd}{\tensor[^{\de}]{K}{^{\de}}}

\newcommand{\ssq}{\sigma^2}

\pgfplotsset{compat=newest}
\pgfplotsset{plot coordinates/math parser=false}

\usepackage{color}

\definecolor{MPG}{RGB}{000,125,122}
\definecolor{dred}{RGB}{130,0,0}
\definecolor{ora}{HTML}{FF9933}
\definecolor{dblu}{RGB}{0,0,130}

\newcommand{\david}[1]{}
\newcommand{\michael}[1]{}
\newcommand{\hennig}[1]{}
\newcommand{\yell}[1]{}

\pgfrealjobname{paper}

\title{Probabilistic ODE Solvers with Runge-Kutta Means}

\author{
Michael Schober\\
MPI for Intelligent Systems\\
T\"ubingen, Germany\\
\texttt{mschober@tue.mpg.de} \\
\And
David Duvenaud \\
Department of Engineering \\
Cambridge University \\
\texttt{dkd23@cam.ac.uk} \\
\And
Philipp Hennig \\
MPI for Intelligent Systems\\
T\"ubingen, Germany\\
\texttt{phennig@tue.mpg.de} \\
}

\nipsfinalcopy 

\begin{document}

\maketitle

\begin{abstract}
  Runge-Kutta methods are the classic family of solvers for ordinary
  differential equations (ODEs), and the basis for the state of the art. Like
  most numerical methods, they return point estimates. We construct a family of
  \emph{probabilistic} numerical methods that instead return a
  Gauss-Markov process defining a probability distribution over the ODE
  solution. In contrast to prior work, we construct this family such that
  posterior means match the outputs of the Runge-Kutta family exactly, thus
  inheriting their proven good properties. Remaining degrees of freedom not
  identified by the match to Runge-Kutta are chosen such that the posterior
  probability measure fits the observed structure of the ODE. Our results shed
  light on the structure of Runge-Kutta solvers from a new direction, provide a
  richer, probabilistic output, have low computational cost, and raise new
  research questions.
\end{abstract}

\section{Introduction}

Differential equations are a basic feature of dynamical systems.
Hence, researchers in machine learning have repeatedly been interested
in both the problem of inferring an ODE description from observed
trajectories of a dynamical system
\cite{graepel2003solving,calderhead2008accelerating,dondelinger2013ode,wang14:GP-ODE},
and its dual, inferring a solution (a trajectory) for an ODE initial
value problem (IVP) \cite{skilling1991bayesian, HennigAISTATS2014,
  Schober,o.13:_bayes_uncer_quant_differ_equat}. Here we address the
latter, classic numerical problem. Runge-Kutta (RK) methods \cite{Runge,
  Kutta} are standard tools for this purpose. Over more than a
century, these algorithms have matured into a very well-understood,
efficient framework \citep{hairer87:_solvin_ordin_differ_equat_i}.

As recently pointed out by \citet{HennigAISTATS2014}, since
Runge-Kutta methods are linear extrapolation methods, their structure
can be emulated by Gaussian process (GP) regression algorithms. Such
an algorithm was envisioned by Skilling in 1991
\cite{skilling1991bayesian}, and the idea has recently attracted both
theoretical \cite{o.13:_bayes_uncer_quant_differ_equat} and practical
\cite{HennigAISTATS2014, Schober} interest. By returning a posterior
probability measure over the solution of the ODE problem, instead of a
point estimate, Gaussian process solvers extend the functionality of
RK solvers in ways that are particularly interesting for machine
learning. Solution candidates can be drawn from the posterior and
marginalized \cite{Schober}. This can allow probabilistic solvers to
stop earlier, and to deal (approximately) with probabilistically
uncertain inputs and problem definitions
\cite{HennigAISTATS2014}. However, current GP ODE solvers do not share
the good theoretical convergence properties of Runge-Kutta
methods. Specifically, they do not have high polynomial order,
explained below.

We construct GP ODE solvers whose posterior mean functions
\emph{exactly} match those of the RK families of first, second and
third order. This yields a probabilistic numerical method which
combines the strengths of Runge-Kutta methods with the additional
functionality of GP ODE solvers. It also provides a new interpretation
of the classic algorithms, raising new conceptual questions.

While our algorithm could be seen as a ``Bayesian'' version of the
Runge-Kutta framework, a philosophically less loaded interpretation is
that, where Runge-Kutta methods fit a single curve (a point estimate)
to an IVP, our algorithm fits a probability distribution over such
potential solutions, such that the mean of this distribution matches
the Runge-Kutta estimate exactly. We find a family of models in the
space of Gaussian process linear extrapolation methods with this
property, and select a member of this family (fix the remaining
degrees of freedom) through statistical estimation.


\section{Background}
\label{sec:background}

\label{sec:ivp-formulation}
An ODE \emph{Initial Value Problem (IVP)} is to find a function
$x(t):\Re\to\Re^N$ such that the ordinary differential equation
$\dot{x} = f(x,t)$ (where $\dot{x}=\de x / \de t$) holds for all $t\in
T = [t_0, t_H]$, and $x(t_0) = x_0$. We assume that a unique solution
exists. To keep notation simple, we will treat $x$ as
scalar-valued; the multivariate extension is straightforward (it
involves $N$ separate GP models, explained in supp.).

\label{sec:rk-methods}
Runge-Kutta methods\footnote{In this work, we only address so-called
  \emph{explicit} RK methods (shortened to ``Runge-Kutta methods'' for
  simplicity). These are the base case of the extensive theory of RK
  methods. Many generalizations can be found in
  \cite{hairer87:_solvin_ordin_differ_equat_i}. Extending the probabilistic
  framework discussed here to the wider Runge-Kutta class is not trivial.}
\cite{Runge, Kutta} are carefully designed linear extrapolation methods
operating on small contiguous subintervals $[t_n,t_{n} + h]\subset T$ of length
$h$. Assume for the moment that
$n = 0$%
. Within $[t_0,t_0 + h]$, an RK method of \emph{stage}
$s$ collects evaluations $y_i=f(\hat{x}_i,t_0+h c_i)$ at $s$
recursively defined input locations, $i=1,\dots,s$, where $\hat{x}_i$
is constructed \emph{linearly} from the previously-evaluated $y_{j<i}$ as
\begin{equation}
\label{eq:xhat}
\hat{x}_i = x_0 + h \sum_{j=1}^{i-1} w_{ij} y_j,
\end{equation}
then returns a single prediction for the solution of the IVP at
$t_0+h$, as $\hat{x}(t_0+h) = x_0 + h \sum_{i=1}^s b_i y_i$ (modern
variants can also construct non-probabilistic error estimates, e.g.~by
combining the same observations into two different RK predictions
\cite{dormand1980family}). In compact form,
\begin{align}\label{eq:rk} 
y_i = f\left(x_0 + h \sum_{j=1}^{i-1} w_{ij} y_j,\;t_0 + h c_i\right),
\q i = 1, \dotsc, s,
\qq \hat{x}(t_0+h) = x_0 + h \sum_{i=1}^s b_i y_i.
\end{align}
$\hat{x}(t_0+h)$~is then taken as the initial value for $t_1 = t_0 + h$ and
the process is repeated until $t_n + h \geq t_H$.

A Runge-Kutta method is thus identified by a lower-triangular matrix
$\mat{W} = \{w_{ij}\}$, and vectors $\vec{c} = [c_1, \dots, c_s]$,
$\vec{b} = [b_1, \dots, b_s]$, often presented compactly in a
\emph{Butcher tableau} \citep{butcher1963coefficients}:
\begin{equation}
\begin{array}{c|ccccc}
c_1 & 0\\
c_2 & w_{21} & 0\\
c_3 & w_{31} & w_{32} & 0\\
\vdots & \vdots & \vdots & \ddots & \;\ddots\\
c_s & w_{s1} & w_{s2} & \cdots & w_{s,s-1} & 0\\
\hline
& b_1 & b_2 & \cdots & b_{s-1} & b_s
\end{array}\notag
\end{equation}
As \citet{HennigAISTATS2014} recently pointed out, the linear
structure of the extrapolation steps in Runge-Kutta methods means that
their algorithmic structure, the Butcher tableau, can be constructed
naturally from a Gaussian process regression method over $x(t)$, where
the $y_i$ are treated as ``observations'' of $\dot x(t_0+hc_i)$ and
the $\hat{x}_i$ are subsequent posterior estimates (more
below). However, proper RK methods have structure that is not
generally reproduced by an arbitrary Gaussian process prior on $x$:
Their distinguishing property is that the approximation $\hat{x}$ and
the Taylor series of the true solution coincide at $t_0 + h$ up to the
$p$-th term---their numerical error is bounded by $\norm{x(t_0 + h) -
  \hat{x}(t_0+h)} \leq K h^{p+1}$ for some constant $K$ (higher orders
are better, because $h$ is assumed to be small). The method is then
said to be \emph{of order $p$}
\cite{hairer87:_solvin_ordin_differ_equat_i}. A method is
\emph{consistent}, if it is of order $p=s$. This is only possible for
$p < 5$ \cite{ceschino1963problemes,shanks1966solutions}. There are no
methods of order $p > s$. High order is a strong desideratum for ODE
solvers, not currently offered by Gaussian process extrapolators.

\begin{table}[bt]
  \begin{tabular*}{\textwidth}{@{\extracolsep{\fill}}ccc}
  \toprule
  $p = 1$ & $p = 2$ & $p = 3$ \\
  \midrule
  $
    \begin{array}{c|ccc}0&0\\\hline&1\end{array}
  $
  &

  $
\begin{array}{c|cc}
0 & 0\\
\alpha & \alpha & 0\\
\hline
\rule{0pt}{11pt}& (1 - \frac{1}{2\alpha}) & \frac{1}{2\alpha}
\end{array}
  $

  &

  $
\begin{array}{c|ccc}
0 & 0\\
u & u & 0\\
v & v - \frac{v(v-u)}{u(2-3u)} & \frac{v(v-u)}{u(2-3u)} & 0\\
\hline
\rule{0pt}{11pt}  & 1 - \frac{2-3v}{6u(u-v)} - \frac{2-3u}{6v(v-u)} & \frac{2-3v}{6u(u-v)} & \frac{2-3u}{6v(v-u)}
\end{array}
  $\\
  \bottomrule
  \end{tabular*}
  \caption{All consistent Runge-Kutta methods of order $p \leq 3$ and
    number of stages $s = p$ (see \citep{hairer87:_solvin_ordin_differ_equat_i}).}
  \label{eq:low-order-rks}
\end{table}


Table~\ref{eq:low-order-rks} lists all consistent methods of order $p \leq 3$
where $s = p$.  For $s=1$, only \emph{Euler's method} (linear
extrapolation) is consistent. For $s=2$, there exists a family of
methods of order $p=2$, parametrized by a single parameter $\alpha \in
(0, 1]$,
where $\alpha = \nicefrac{1}{2}$ and $\alpha=1$ mark the
\emph{midpoint rule} and \emph{Heun's method}, respectively.
For $s=3$, third order methods are parameterized by two variables $u,
v \in (0, 1]$.

\label{sec:gps-for-odes}
\emph{Gaussian processes (GPs)} are well-known in the NIPS community,
so we omit an introduction. We will use the standard notation $\mu :
\Re \to \Re$ for the mean function, and $k : \Re \times \Re
\to \Re$ for the covariance function; $k_{UV}$ for Gram matrices
of kernel values $k(u_i,v_j)$, and analogous for the mean function:
$\mu_T = [\mu(t_1),\dots,\mu(t_N)]$. A GP prior $p(x)=\GP(x; \mu, k)$
and observations $(T, Y) = \{(t_1, y_1), \dots, (t_s, y_s)\}$ having
likelihood $\N(Y;x_T,\Lambda)$ give rise to a posterior $\GP^s(x;
\mu^s, k^s)$ with
\begin{equation}
\mu^s _t = \mu_t + k_{tT} (k_{TT}+\Lambda) ^{-1}(Y-\mu_T) \qq\text{and}\qq
k^s _{uv} = k_{uv} - k_{uT}(k_{TT}+\Lambda)^{-1}k_{Tv}.
\label{eq:posterior}
\end{equation}
GPs are closed under linear maps. In particular, the joint distribution over
$x$ and its derivative is
\begin{gather}
  \label{eq:2}
  p\left[
    \begin{pmatrix}
      x \\ \dot x
    \end{pmatrix}
\right] = \GP\left[
  \begin{pmatrix}
      x \\ \dot x
    \end{pmatrix};
    \begin{pmatrix}
      \mu \\ \mu^{\de}
    \end{pmatrix},
    \begin{pmatrix}
      k & \kd \\ \dk & \dkd
    \end{pmatrix}
  \right]\\
  \text{with}\qq \mu^\de = \frac{\de \mu(t)}{\de t},\q \kd = \frac{\de k(t,t')}{\de t'}, \q \dk
  = \frac{\de k(t,t')}{\de t}, \q \dkd = \frac{\de^2 k(t,t')}{\de t\de t'}.
\end{gather}
A recursive algorithm analogous to RK methods can be constructed
\citep{skilling1991bayesian,HennigAISTATS2014} by setting the prior
mean to the constant $\mu(t)=x_0$, then recursively estimating
$\hat{x}_i$ in some form from the current posterior over $x$. The
choice in \citep{HennigAISTATS2014} is to set $\hat{x}_i =
\mu^i(t_0+hc_i)$.  ``Observations'' $y_i=f(\hat{x}_i,t_0+hc_i)$ are
then incorporated with likelihood
$p(y_i\g x) = \N(y_i;\dot x(t_0+hc_i),\Lambda)$.
This recursively gives estimates 
\begin{equation}
  \label{eq:3}
  \hat{x}(t_0+hc_i) = x_0 + \sum_{j=1}^{i-1}\sum_{\ell=1}^{i-1} \kd(t_0+hc_i,
  t_0+hc_\ell) (\dKd +\Lambda)^{-1} _{\ell j} y_j = x_0 + h\sum_j w_{ij}y_j, 
\end{equation}
with $\dKd_{ij} = \dkd(t_0+hc_i,t_0 + hc_j)$. The final prediction
is the posterior mean at this point:
\begin{equation}
  \label{eq:4}
  \hat{x}(t_0+h) = x_0 + \sum_{i=1}^s \sum_{j=1}^s \kd(t_0+h, t_0+hc_j)
  (\dKd+\Lambda)^{-1}_{ji} y_i = x_0 + h\sum_{i} ^s b_i y_i.
\end{equation}
\michael{Double check these last two formulas. Indices might be off!}

\section{Results}
\label{sec:rel-gp-rk}
The described GP ODE estimate shares the algorithmic structure
of RK methods (i.e.~they both use weighted sums of the constructed
estimates to extrapolate). However, in RK methods, weights and
evaluation positions are found by careful analysis of the Taylor
series of $f$, such that low-order terms cancel. In GP ODE solvers
they arise, perhaps more naturally but also with less structure, by
the choice of the $c_i$ and the kernel. In
previous work \cite{HennigAISTATS2014, Schober}, both were chosen ad
hoc, with no guarantee of convergence order. In fact, as is shown in
the supplements, the choices in these two works---square-exponential
kernel with finite length-scale%
, evaluations at the predictive mean---do not even give the
first order convergence of Euler's method. Below we present three
specific regression models based on integrated Wiener covariance
functions and specific evaluation points. Each model is the improper
limit of a Gauss-Markov process, such that the posterior distribution
after $s$ evaluations is a proper Gaussian process, and the posterior
mean function at $t_0+h$ coincides \emph{exactly} with the Runge-Kutta
estimate. We will call these methods, which give a probabilistic
interpretation to RK methods and extend them to return probability
distributions, \emph{Gauss-Markov-Runge-Kutta (GMRK) methods}, because
they are based on Gauss-Markov priors and yield Runge-Kutta
predictions.

\subsection{Design choices and desiderata for a probabilistic ODE solver}
\label{sec:arrangement}

Although we are not the first to attempt constructing an ODE solver
that returns a probability distribution, open questions still remain
about what, exactly, the properties of such a probabilistic numerical
method should be. \citet{o.13:_bayes_uncer_quant_differ_equat}
previously made the case that Gaussian measures are uniquely suited
because solution spaces of ODEs are Banach spaces, and provided
results on consistency. Above, we added the desideratum for the
posterior mean to have high order, i.e. to reproduce the Runge-Kutta
estimate. Below, three additional issues become apparent:

\paragraph{Motivation of evaluation points}
\label{sec:motiv-eval-points}

Both \citet{skilling1991bayesian} and \citet{HennigAISTATS2014}
propose to put the ``nodes'' $\hat{x}(t_0+hc_i)$ at the current
posterior mean of the belief. We will find that this can be made
consistent with the order requirement for the RK methods of first and
second order. However, our third-order methods will be forced to use a
node $\hat{x}(t_0+hc_i)$ that, albeit lying along a function~$w(t)$ in the
reproducing kernel
Hilbert space associated with the posterior GP covariance function, is
not the mean function itself. It will remain open whether the algorithm can be
amended to remove this blemish. However, as the nodes do not enter the GP
regression formulation, their choice does not directly affect the
probabilistic interpretation.

\paragraph{Extension beyond the first extrapolation interval}
\label{sec:extens-beyond-first}

Importantly, the Runge-Kutta argument for convergence order only holds
strictly for the first extrapolation interval $[t_0,t_0+h]$. From the
second interval onward, the RK step solves an estimated IVP, and
begins to accumulate a global estimation error not bounded by the
convergence order (an effect termed ``Lady Windermere's fan'' by
Wanner \cite{hairer12:_numer}). Should a probabilistic solver aim to
faithfully reproduce this imperfect chain of RK solvers, or rather try
to capture the accumulating global error?  We investigate both options
below.

\paragraph{Calibration of uncertainty}
\label{sec:calibr-error-meas}

A question easily posed but hard to answer is what it means for the
probability distribution returned by a probabilistic method to be well
calibrated. For our Gaussian case, requiring RK order in the posterior
mean determines all but one degree of freedom of an answer. The
remaining parameter is the output scale of the kernel, the ``error
bar'' of the estimate. We offer a relatively simple statistical
argument below that fits this parameter based on observed values of
$f$.

We can now proceed to the main results. In the following, we consider
extrapolation algorithms based on Gaussian process priors with
vanishing prior mean function, noise-free observation model
($\Lambda=0$ in Eq.~(\ref{eq:posterior})).  All covariance functions
in question are integrals over the kernel $k^0(\tilde{t},\tilde{t}') =
\ssq\min(\tilde{t}-\tau,\tilde{t}' -\tau)$ (parameterized by scale
$\ssq>0$ and off-set $\tau\in\Re$; valid on the domain
$\tilde{t},\tilde{t}'>\tau$), the covariance of the Wiener process
\citep{wiener1950extrapolation}. Such integrated Wiener processes are
Gauss-Markov processes, of increasing order, so inference in these
methods can be performed by filtering, at linear cost
\citep{sarkka2013bayesian}. We will use the shorthands
$t=\tilde{t}-\tau$ and $t' = \tilde{t}' - \tau$ for inputs shifted by
$\tau$.




\begin{figure}[t]
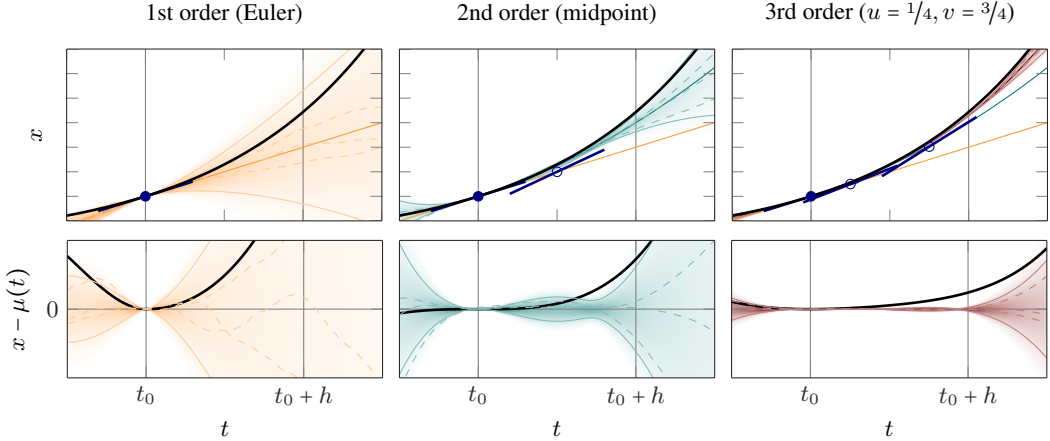

  \centering
  \footnotesize
  \setlength\figurewidth{0.3\textwidth}%
  \setlength\figureheight{0.1\textheight}%
  \mbox{%
    \hspace{1.2mm}%
  \input{figures/WP1-theorem.tikz}%
\hspace{-0.9mm}%
  \input{figures/WP2-theorem.tikz}%
\hspace{-0.9mm}%
  \input{figures/WP3-theorem.tikz}%
  }\\
  \setlength\figureheight{0.08\textheight}%
  \mbox{%
    \hspace{-4mm}
  \input{figures/WP1-theorem-std.tikz}%
  \input{figures/WP2-theorem-std.tikz}%
  \input{figures/WP3-theorem-std.tikz}%
  }\\
  \caption{{\bf Top:} Conceptual sketches. Prior mean in gray. Initial
    value at $t_0=1$ (filled blue). Gradient evaluations (empty blue
    circles, lines). Posterior (means) after first, second and
    third gradient observation in orange, green and red
    respectively. Samples from the final posterior as dashed
    lines. Since, for the second and third-order methods, only the
    final prediction is a proper probability distribution, for
    intermediate steps only mean functions are shown. True solution to
    (linear) ODE in black. {\bf Bottom:} For better visibility, same
    data as above, minus final posterior mean.}
  \label{fig:theorems}
\end{figure}

\subsection{Gauss-Markov methods matching Euler's method}
\label{sec:euler-method}
\begin{thm}
\label{thm:euler}
The once-integrated Wiener process prior $p(x)=\GP(x;0,k^1)$ with 
\begin{equation}
  k^1(t,t') = \iint_\tau ^{\tilde{t},\tilde{t}'} k^0(u,v) du\, dv =  \ssq \left(\frac{\min\nolimits^3(t,t')}{3}
    +\abs{t - t'} \frac{\min\nolimits^2(t,t')}{2}\right)
\end{equation}
choosing evaluation nodes at the posterior mean gives rise to Euler's
method.
\end{thm}
\begin{proof}
  We show that the corresponding Butcher tableau from Table
  \ref{eq:low-order-rks} holds. After ``observing'' the initial value,
  the second observation $y_1$, constructed by evaluating $f$ at the
  posterior mean at $t_0$, is
\begin{equation}
\label{eq:proof-y1}
  y_1 = f\left( \mu_{\vert x_0}(t_0), t_0 \right)
  = f\left(\frac{k(t_0,t_0)}{k(t_0,t_0)} x_0, t_0\right) = f(x_0, t_0),
\end{equation}
directly from the definitions. The posterior mean after incorporating
$y_1$ is
\begin{align}
  \mu_{\vert x_0, y_1}(t_0 + h) &= \begin{bmatrix}k(t_0 + h,t_0)
    & \kd(t_0 + h, t_0)\end{bmatrix}
  \begin{bmatrix}k(t_0,t_0) & \kd(t_0, t_0)\\
    \kd(t_0, t_0) & \dkd(t_0, t_0) \end{bmatrix}^{-1}
  \begin{pmatrix}x_0\\y_1\end{pmatrix}
  = x_0 + hy_1. \label{eq:euler}
\end{align}
An explicit linear algebraic derivation is available in the
supplements.
\end{proof}

\subsection{Gauss-Markov methods matching all Runge-Kutta methods of
  second order}
\label{sec:2nd-order-methods}

Extending to second order is not as straightforward as integrating the
Wiener process a second time. The theorem below shows that this only
works after moving the onset $-\tau$ of the process towards
infinity. Fortunately, this limit still leads to a proper posterior
probability distribution.

\begin{thm}
\label{thm:2nd-order}
  Consider the twice-integrated Wiener process prior $p(x)=\GP(x;0,k^2)$ with 
  \begin{equation}
    \label{eq:5}
    k^2(t,t') =\iint_\tau ^{\tilde{t},\tilde{t}'} k^1(u,v) du\, dv  = \ssq \left(\frac{\min\nolimits^5(t,t')}{20} 
      +\frac{\abs{t-t'}}{12} \left((t+t')\min\nolimits^3(t,t')
        - \frac{\min\nolimits^4(t,t')}{2}
      \right)\right).
  \end{equation}
  Choosing evaluation nodes at the posterior mean gives rise to the RK
  family of second order methods in the limit of $\tau\to\infty$.
\end{thm}
(The twice-integrated Wiener process is a proper Gauss-Markov process
for all finite values of $\tau$ and $\tilde{t},\tilde{t}'>0$. In the
limit of $\tau\to \infty$, it turns into an improper prior of infinite
local variance.)
\begin{proof}
  The proof is analogous to the previous one. We need to show all
  equations given by the Butcher tableau and choice of parameters hold
  for any choice of $\alpha$. The constraint for $y_1$ holds trivially
  as in Eq.~\eqref{eq:proof-y1}. Because $y_2 = f(x_0 + h\alpha y_1,
  t_0 + h\alpha)$, we need to show $\mu_{\vert x_0, y_1}(t_0 +
  h\alpha) = x_0 + h\alpha y_1$. Therefore, let $\alpha \in (0, 1]$
  arbitrary but fixed:
\begin{align}
\mu_{\vert x_0, y_1}(t_0 + h\alpha) &= \begin{bmatrix}k(t_0 + h,t_0)
    & \kd(t_0 + h, t_0)\end{bmatrix}
  \begin{bmatrix}k(t_0,t_0) & \kd(t_0, t_0)\\
    \dk(t_0, t_0) & \dkd(t_0, t_0) \end{bmatrix}^{-1}
  \begin{pmatrix}x_0\\y_1\end{pmatrix}\notag\\
&= \begin{bmatrix} \frac{t_0^3\left(10(h\alpha)^2+15h\alpha t_0+6t_0^2\right)}{120} 
 & \frac{t_0^2\left(6(h\alpha)^2 + 8h \alpha t_0 + 3t_0^2\right)}{24} \end{bmatrix}
  \begin{bmatrix} \nicefrac{t_0^5}{20} & \nicefrac{t_0^4}{8}\\
    \nicefrac{t_0^4}{8} & \nicefrac{t_0^3}{3} \end{bmatrix}^{-1}
  \begin{pmatrix}x_0\\y_1\end{pmatrix}\notag\\
&=\begin{bmatrix}1 - \frac{10(h\alpha)^2}{3t_0^2} & 
                 h\alpha + \frac{2(h\alpha)^2}{t_0} \end{bmatrix}
  \begin{pmatrix}x_0\\y_1\end{pmatrix}\notag\\
&\xrightarrow[\tau \to \infty]{} x_0 + h\alpha y_1
\end{align}
As $t_0 = \tilde{t}_0 - \tau$, the mismatched terms vanish for $\tau \to
\infty$.  Finally, extending the vector and matrix with one more entry, a
lengthy computation shows that $\lim_{\tau \to \infty} \mu_{\vert x_0, y_1,
  y_2}(t_0 + h) = x_0 + h (1 - \nicefrac{1}{2\alpha}) y_1 +
\nicefrac{h}{2\alpha} y_2$ also holds, analogous to
Eq.~\eqref{eq:euler}. Omitted details can be found in the supplements. They
also include the final-step posterior covariance. Its finite values mean that
this posterior indeed defines a proper GP.
\end{proof}


\subsection{A Gauss-Markov method matching Runge-Kutta methods of third order}
\label{sec:gauss-markov-method}

Moving from second to third order, additionally to the limit towards
an improper prior, also requires a departure from the policy of placing
extrapolation nodes at the posterior mean.

\begin{thm}
  Consider the thrice-integrated Wiener process prior
  $p(x)=\GP(x;0,k^3)$ with
  \begin{equation}
    \label{eq:7}
    \begin{split}
      k^3(t,t') &= \iint_\tau ^{\tilde{t},\tilde{t}'} k^2(u,v) du\, dv\\
      &= \ssq \left(\frac{\min\nolimits^7(t,t')}{252} +
        \frac{\abs{t-t'}\min\nolimits^4(t,t')}{720} \left(
          5\max\nolimits^2(t,t') + 2tt' +
          3\min\nolimits^2(t,t')\right) \right).
    \end{split}
  \end{equation}
  Evaluating twice at the posterior mean and a third time at a
  specific element of the posterior covariance functions' RKHS gives
  rise to the entire family of RK methods of third order, in the limit
  of $\tau\to\infty$.
\end{thm}
\begin{proof}
  The proof progresses entirely analogously as in
  Theorems~\ref{thm:euler} and \ref{thm:2nd-order}, with one exception
  for the term where the mean does not match the RK weights exactly.
  This is the case for $y_3 = x_0 + h[(v -
  \nicefrac{v(v-u)}{u(2-3u)})y_1 + \nicefrac{v(v-u)}{u(2-3u)}y_2]$
  (see Table \ref{eq:low-order-rks}). The weights of $Y$ which give
  the posterior mean at this point are given by $kK^{-1}$
  (cf.~Eq.~\eqref{eq:posterior}, which, in the limit, has value (see
  supp.):
  \begin{align}
    \lim_{\tau \to \infty} &\begin{bmatrix}k(t_0 + hv,t_0) & 
      \kd(t_0 + hv, t_0) & 
      \kd(t_0 + hv, t_0 + hu)
    \end{bmatrix}
    K^{-1}
    \notag\\
    = &\begin{bmatrix} 1 & h(v - \frac{v^2}{2u}) &
      h\frac{v^2}{2u}\end{bmatrix} \notag\\
    = &\begin{bmatrix} 1 & h\left(v - \frac{v(v-u)}{u(2-3u)} \mathbf{-
          \frac{v(3v-2)}{2(3u-2)}}\right) &
      h\left(\frac{v(v-u)}{u(2-3u)} \mathbf{+
          \frac{v(3v-2)}{2(3u-2)}}\right)\end{bmatrix}\notag\\
    = &\begin{bmatrix} 1 & h\left(v - \frac{v(v-u)}{u(2-3u)}\right) &
      h\left(\frac{v(v-u)}{u(2-3u)}\right)\end{bmatrix}
    + \begin{bmatrix} 0 & -h\frac{v(3v-2)}{2(3u-2)} &
      h\frac{v(3v-2)}{2(3u-2)}\end{bmatrix}\label{eq:rkhs-element}
  \end{align}
  This means that the final RK evaluation node does not lie at the
  posterior mean of the regressor. However, it can be produced by
  adding a correction term $w(v) = \mu(v) + \varepsilon(v)(y_2-y_1)$
  where
\begin{equation}
  \label{eq:corr-term}
  \varepsilon(v)=\frac{v}{2} \frac{3v - 2}{3u - 2}  
\end{equation}
is a second-order polynomial in $v$. Since $k$ is of third or higher
order in $v$ (depending on the value of $u$), $w$ can be written as an
element of the thrice integrated Wiener process' RKHS
\cite[\textsection 6.1]{RasmussenWilliams}. Importantly, the final
extrapolation weights $b$ under the limit of the Wiener process prior
again match the RK weights exactly, regardless of how $y_3$ is
constructed.
\end{proof}

We note in passing that Eq.~\eqref{eq:corr-term} vanishes for $v =
\nicefrac{2}{3}$. For this choice, the RK observation $y_2$ is
generated exactly at the posterior mean of the Gaussian
process. Intriguingly, this is also the value for $\alpha$ for which
the posterior variance at $t_0 + h$ is minimized.

\subsection{Choosing the output scale}
\label{sec:choos-outp-scale}

The above theorems have shown that the first three families of Runge-Kutta
methods can be constructed from repeatedly integrated Wiener process priors,
giving a strong argument for the use of such priors in probabilistic numerical
methods. However, requiring this match to a specific Runge-Kutta family in
itself does not yet uniquely identify a particular kernel to be used: The
posterior mean of a Gaussian process arising from noise-free observations is
independent of the output scale (in our notation: $\ssq$) of the covariance
function (this can also be seen by inspecting Eq.~(\ref{eq:posterior})). Thus,
the parameter $\ssq$ can be chosen independent of the other parts of the
algorithm, without breaking the match to Runge-Kutta. Several algorithms using
the observed values of $f$ to choose $\ssq$ without major cost overhead have
been proposed in the regression community before
\cite[e.g.][]{shumway1982approach,ghahramani1996parameter}. For this particular
model an even more basic rule is possible: A simple derivation shows that, in
all three families of methods defined above, the posterior belief over
$\nicefrac{\de^s x}{\de t^s}$ is a Wiener process, and the posterior mean
function over the $s$-th derivative after all $s$ steps is a constant
function. The Gaussian model implies that the expected distance of this
function from the (zero) prior mean should be the marginal standard deviation
$\sqrt{\ssq}$. We choose $\ssq$ such that this property is met, by setting
$\ssq=\left[\nicefrac{\de^s \mu_s(t)}{\de t^s}\right]^2$.

Figure~\ref{fig:theorems} shows conceptual sketches highlighting the
structure of GMRK methods. Interestingly, in both the second- and
third-order families, our proposed priors are improper, so the solver
can not actually return a probability distribution until after the
observation of all $s$ gradients in the RK step.

\paragraph{Some observations}
\label{sec:summary}

We close the main results by highlighting some non-obvious
aspects. First, it is intriguing that higher convergence order results
from repeated integration of Wiener processes. This repeated
integration simultaneously adds to and weakens certain prior
assumptions in the implicit (improper) Wiener prior: $s$-times
integrated Wiener processes have marginal variance $k^s(t,t)\propto
t^{2s+1}$. Since many ODEs (e.g. linear ones) have solution paths of
values $\mathcal{O}(\exp(t))$, it is tempting to wonder whether there
exists a limit process of ``infinitely-often integrated'' Wiener
processes giving natural coverage to this domain (the results on a
linear ODE in Figure \ref{fig:theorems} show how the polynomial
posteriors cannot cover the exponentially diverging true solution). In
this context, it is also noteworthy that $s$-times integrated Wiener
priors incorporate the lower-order results for $s'<s$, so
``highly-integrated'' Wiener kernels can be used to match finite-order
Runge-Kutta methods. Simultaneously, though, sample paths from an
$s$-times integrated Wiener process are almost surely $s$-times
differentiable. So it seems likely that achieving good performance
with a Gauss-Markov-Runge-Kutta solver requires trading off the good
marginal variance coverage of high-order Markov models
(i.e.~repeatedly integrated Wiener processes) against modelling
non-smooth solution paths with lower degrees of integration. We leave
this very interesting question for future work.

\section{Experiments}
\label{sec:experiments}

\begin{figure}
  \centering
  \footnotesize
  \setlength\figurewidth{0.2679\textwidth}
  \setlength\figureheight{0.12\textheight}
  \mbox{\qquad \qquad Na\"ive chaining \qquad \qquad \qquad \qquad \quad Smoothing \quad \qquad \qquad \qquad Probabilistic continuation}
  \mbox{%
    \hspace{-1mm}%
%
%
%
\definecolor{mycolor1}{rgb}{0.49060,0.00000,0.00000}%
\definecolor{mycolor2}{rgb}{0.00000,0.47170,0.46040}%
\definecolor{mycolor3}{rgb}{0.00000,0.00000,0.50900}%
\begin{tikzpicture}

\begin{axis}[%
width=\figurewidth,
height=\figureheight,
axis on top,
scale only axis,
xmin=-0.0100502512562814,
xmax=4.01005025125628,
xtick={0,1,2,3,4},
xticklabels={\empty},
ymin=0.0977386934673367,
ymax=1.00226130653266,
ylabel={$x$},
xlabel near ticks,
ylabel near ticks
]
\addplot [forget plot] graphics [xmin=-0.0100502512562814,xmax=4.01005025125628,ymin=0.0977386934673367,ymax=1.00226130653266] {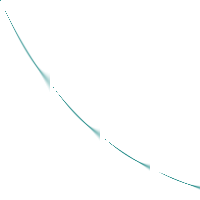};
\addplot [color=gray,solid,forget plot]
  table[row sep=crcr]{figures/2nd-order-naive_noticks-1.tsv};
\addplot [color=gray,solid,forget plot]
  table[row sep=crcr]{figures/2nd-order-naive_noticks-2.tsv};
\addplot [color=gray,solid,forget plot]
  table[row sep=crcr]{figures/2nd-order-naive_noticks-3.tsv};
\addplot [color=mycolor1,solid,forget plot]
  table[row sep=crcr]{figures/2nd-order-naive_noticks-4.tsv};
\addplot [color=white!50!mycolor1,dashed,forget plot]
  table[row sep=crcr]{figures/2nd-order-naive_noticks-5.tsv};
\addplot [color=white!50!mycolor1,dashed,forget plot]
  table[row sep=crcr]{figures/2nd-order-naive_noticks-6.tsv};
\addplot [color=mycolor2,solid,forget plot]
  table[row sep=crcr]{figures/2nd-order-naive_noticks-7.tsv};
\addplot [color=white!50!mycolor2,dashed,forget plot]
  table[row sep=crcr]{figures/2nd-order-naive_noticks-8.tsv};
\addplot [color=white!50!mycolor2,dashed,forget plot]
  table[row sep=crcr]{figures/2nd-order-naive_noticks-9.tsv};
\addplot [color=mycolor2,solid,forget plot]
  table[row sep=crcr]{figures/2nd-order-naive_noticks-10.tsv};
\addplot [color=white!50!mycolor2,dashed,forget plot]
  table[row sep=crcr]{figures/2nd-order-naive_noticks-11.tsv};
\addplot [color=white!50!mycolor2,dashed,forget plot]
  table[row sep=crcr]{figures/2nd-order-naive_noticks-12.tsv};
\addplot [color=mycolor2,solid,forget plot]
  table[row sep=crcr]{figures/2nd-order-naive_noticks-13.tsv};
\addplot [color=white!50!mycolor2,dashed,forget plot]
  table[row sep=crcr]{figures/2nd-order-naive_noticks-14.tsv};
\addplot [color=white!50!mycolor2,dashed,forget plot]
  table[row sep=crcr]{figures/2nd-order-naive_noticks-15.tsv};
\addplot [color=blue,mark size=1.8pt,only marks,mark=*,mark options={solid,fill=mycolor3,draw=mycolor3},forget plot]
  table[row sep=crcr]{figures/2nd-order-naive_noticks-16.tsv};
\addplot [color=blue,mark size=1.8pt,only marks,mark=o,mark options={solid,draw=mycolor3},forget plot]
  table[row sep=crcr]{figures/2nd-order-naive_noticks-17.tsv};
\addplot [color=blue,mark size=1.8pt,only marks,mark=*,mark options={solid,fill=mycolor3,draw=mycolor3},forget plot]
  table[row sep=crcr]{figures/2nd-order-naive_noticks-18.tsv};
\addplot [color=blue,mark size=1.8pt,only marks,mark=o,mark options={solid,draw=mycolor3},forget plot]
  table[row sep=crcr]{figures/2nd-order-naive_noticks-19.tsv};
\addplot [color=blue,mark size=1.8pt,only marks,mark=*,mark options={solid,fill=mycolor3,draw=mycolor3},forget plot]
  table[row sep=crcr]{figures/2nd-order-naive_noticks-20.tsv};
\addplot [color=blue,mark size=1.8pt,only marks,mark=o,mark options={solid,draw=mycolor3},forget plot]
  table[row sep=crcr]{figures/2nd-order-naive_noticks-21.tsv};
\addplot [color=blue,mark size=1.8pt,only marks,mark=*,mark options={solid,fill=mycolor3,draw=mycolor3},forget plot]
  table[row sep=crcr]{figures/2nd-order-naive_noticks-22.tsv};
\addplot [color=blue,mark size=1.8pt,only marks,mark=o,mark options={solid,draw=mycolor3},forget plot]
  table[row sep=crcr]{figures/2nd-order-naive_noticks-23.tsv};
\addplot [color=black,solid,line width=1.0pt,forget plot]
  table[row sep=crcr]{figures/2nd-order-naive_noticks-24.tsv};
\end{axis}
\end{tikzpicture}%
%
\hspace{5.1mm}%
%
%
%
\definecolor{mycolor1}{rgb}{0.00000,0.47170,0.46040}%
\definecolor{mycolor2}{rgb}{0.49060,0.00000,0.00000}%
\definecolor{mycolor3}{rgb}{0.00000,0.00000,0.50900}%
\begin{tikzpicture}

\begin{axis}[%
width=\figurewidth,
height=\figureheight,
axis on top,
scale only axis,
xmin=-0.0100502512562814,
xmax=4.01005025125628,
xtick={0,1,2,3,4},
xticklabels={\empty},
ymin=0.0977386934673367,
ymax=1.00226130653266,
ytick={0.1,0.2,0.3,0.4,0.5,0.6,0.7,0.8,0.9,1},
yticklabels={\empty},
xlabel near ticks,
ylabel near ticks
]
\addplot [forget plot] graphics [xmin=-0.0100502512562814,xmax=4.01005025125628,ymin=0.0977386934673367,ymax=1.00226130653266] {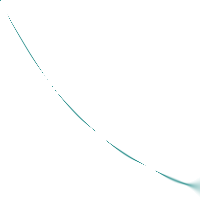};
\addplot [color=gray,solid,forget plot]
  table[row sep=crcr]{figures/2nd-order-smoothing_noticks-1.tsv};
\addplot [color=gray,solid,forget plot]
  table[row sep=crcr]{figures/2nd-order-smoothing_noticks-2.tsv};
\addplot [color=gray,solid,forget plot]
  table[row sep=crcr]{figures/2nd-order-smoothing_noticks-3.tsv};
\addplot [color=mycolor1,solid,forget plot]
  table[row sep=crcr]{figures/2nd-order-smoothing_noticks-4.tsv};
\addplot [color=white!50!mycolor1,dashed,forget plot]
  table[row sep=crcr]{figures/2nd-order-smoothing_noticks-5.tsv};
\addplot [color=white!50!mycolor1,dashed,forget plot]
  table[row sep=crcr]{figures/2nd-order-smoothing_noticks-6.tsv};
\addplot [color=mycolor1,solid,forget plot]
  table[row sep=crcr]{figures/2nd-order-smoothing_noticks-7.tsv};
\addplot [color=white!50!mycolor1,dashed,forget plot]
  table[row sep=crcr]{figures/2nd-order-smoothing_noticks-8.tsv};
\addplot [color=white!50!mycolor1,dashed,forget plot]
  table[row sep=crcr]{figures/2nd-order-smoothing_noticks-9.tsv};
\addplot [color=mycolor2,solid,forget plot]
  table[row sep=crcr]{figures/2nd-order-smoothing_noticks-10.tsv};
\addplot [color=white!50!mycolor2,dashed,forget plot]
  table[row sep=crcr]{figures/2nd-order-smoothing_noticks-11.tsv};
\addplot [color=white!50!mycolor2,dashed,forget plot]
  table[row sep=crcr]{figures/2nd-order-smoothing_noticks-12.tsv};
\addplot [color=mycolor1,solid,forget plot]
  table[row sep=crcr]{figures/2nd-order-smoothing_noticks-13.tsv};
\addplot [color=white!50!mycolor1,dashed,forget plot]
  table[row sep=crcr]{figures/2nd-order-smoothing_noticks-14.tsv};
\addplot [color=white!50!mycolor1,dashed,forget plot]
  table[row sep=crcr]{figures/2nd-order-smoothing_noticks-15.tsv};
\addplot [color=blue,mark size=1.8pt,only marks,mark=*,mark options={solid,fill=mycolor3,draw=mycolor3},forget plot]
  table[row sep=crcr]{figures/2nd-order-smoothing_noticks-16.tsv};
\addplot [color=blue,mark size=1.8pt,only marks,mark=o,mark options={solid,draw=mycolor3},forget plot]
  table[row sep=crcr]{figures/2nd-order-smoothing_noticks-17.tsv};
\addplot [color=blue,mark size=1.8pt,only marks,mark=*,mark options={solid,fill=mycolor3,draw=mycolor3},forget plot]
  table[row sep=crcr]{figures/2nd-order-smoothing_noticks-18.tsv};
\addplot [color=blue,mark size=1.8pt,only marks,mark=o,mark options={solid,draw=mycolor3},forget plot]
  table[row sep=crcr]{figures/2nd-order-smoothing_noticks-19.tsv};
\addplot [color=blue,mark size=1.8pt,only marks,mark=*,mark options={solid,fill=mycolor3,draw=mycolor3},forget plot]
  table[row sep=crcr]{figures/2nd-order-smoothing_noticks-20.tsv};
\addplot [color=blue,mark size=1.8pt,only marks,mark=o,mark options={solid,draw=mycolor3},forget plot]
  table[row sep=crcr]{figures/2nd-order-smoothing_noticks-21.tsv};
\addplot [color=blue,mark size=1.8pt,only marks,mark=*,mark options={solid,fill=mycolor3,draw=mycolor3},forget plot]
  table[row sep=crcr]{figures/2nd-order-smoothing_noticks-22.tsv};
\addplot [color=blue,mark size=1.8pt,only marks,mark=o,mark options={solid,draw=mycolor3},forget plot]
  table[row sep=crcr]{figures/2nd-order-smoothing_noticks-23.tsv};
\addplot [color=black,solid,line width=1.0pt,forget plot]
  table[row sep=crcr]{figures/2nd-order-smoothing_noticks-24.tsv};
\end{axis}
\end{tikzpicture}%
%
\hspace{4.8mm}%
%
%
%
\definecolor{mycolor1}{rgb}{0.00000,0.47170,0.46040}%
\definecolor{mycolor2}{rgb}{0.49060,0.00000,0.00000}%
\definecolor{mycolor3}{rgb}{0.00000,0.00000,0.50900}%
\begin{tikzpicture}

\begin{axis}[%
width=\figurewidth,
height=\figureheight,
axis on top,
scale only axis,
xmin=-0.0100502512562814,
xmax=4.01005025125628,
xtick={0,1,2,3,4},
xticklabels={\empty},
ymin=0.0977386934673367,
ymax=1.00226130653266,
ytick={0.1,0.2,0.3,0.4,0.5,0.6,0.7,0.8,0.9,1},
yticklabels={\empty},
xlabel near ticks,
ylabel near ticks
]
\addplot [forget plot] graphics [xmin=-0.0100502512562814,xmax=4.01005025125628,ymin=0.0977386934673367,ymax=1.00226130653266] {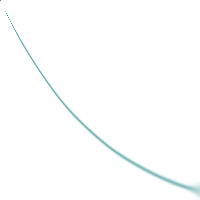};
\addplot [color=gray,solid,forget plot]
  table[row sep=crcr]{figures/2nd-order-continuation_noticks-1.tsv};
\addplot [color=gray,solid,forget plot]
  table[row sep=crcr]{figures/2nd-order-continuation_noticks-2.tsv};
\addplot [color=gray,solid,forget plot]
  table[row sep=crcr]{figures/2nd-order-continuation_noticks-3.tsv};
\addplot [color=mycolor1,solid,forget plot]
  table[row sep=crcr]{figures/2nd-order-continuation_noticks-4.tsv};
\addplot [color=white!50!mycolor1,dashed,forget plot]
  table[row sep=crcr]{figures/2nd-order-continuation_noticks-5.tsv};
\addplot [color=white!50!mycolor1,dashed,forget plot]
  table[row sep=crcr]{figures/2nd-order-continuation_noticks-6.tsv};
\addplot [color=mycolor1,solid,forget plot]
  table[row sep=crcr]{figures/2nd-order-continuation_noticks-7.tsv};
\addplot [color=white!50!mycolor1,dashed,forget plot]
  table[row sep=crcr]{figures/2nd-order-continuation_noticks-8.tsv};
\addplot [color=white!50!mycolor1,dashed,forget plot]
  table[row sep=crcr]{figures/2nd-order-continuation_noticks-9.tsv};
\addplot [color=mycolor2,solid,forget plot]
  table[row sep=crcr]{figures/2nd-order-continuation_noticks-10.tsv};
\addplot [color=white!50!mycolor2,dashed,forget plot]
  table[row sep=crcr]{figures/2nd-order-continuation_noticks-11.tsv};
\addplot [color=white!50!mycolor2,dashed,forget plot]
  table[row sep=crcr]{figures/2nd-order-continuation_noticks-12.tsv};
\addplot [color=mycolor1,solid,forget plot]
  table[row sep=crcr]{figures/2nd-order-continuation_noticks-13.tsv};
\addplot [color=white!50!mycolor1,dashed,forget plot]
  table[row sep=crcr]{figures/2nd-order-continuation_noticks-14.tsv};
\addplot [color=white!50!mycolor1,dashed,forget plot]
  table[row sep=crcr]{figures/2nd-order-continuation_noticks-15.tsv};
\addplot [color=blue,mark size=1.8pt,only marks,mark=o,mark options={solid,draw=mycolor3},forget plot]
  table[row sep=crcr]{figures/2nd-order-continuation_noticks-16.tsv};
\addplot [color=blue,mark size=1.8pt,only marks,mark=o,mark options={solid,draw=mycolor3},forget plot]
  table[row sep=crcr]{figures/2nd-order-continuation_noticks-17.tsv};
\addplot [color=blue,mark size=1.8pt,only marks,mark=o,mark options={solid,draw=mycolor3},forget plot]
  table[row sep=crcr]{figures/2nd-order-continuation_noticks-18.tsv};
\addplot [color=blue,mark size=1.8pt,only marks,mark=o,mark options={solid,draw=mycolor3},forget plot]
  table[row sep=crcr]{figures/2nd-order-continuation_noticks-19.tsv};
\addplot [color=blue,mark size=1.8pt,only marks,mark=o,mark options={solid,draw=mycolor3},forget plot]
  table[row sep=crcr]{figures/2nd-order-continuation_noticks-20.tsv};
\addplot [color=blue,mark size=1.8pt,only marks,mark=o,mark options={solid,draw=mycolor3},forget plot]
  table[row sep=crcr]{figures/2nd-order-continuation_noticks-21.tsv};
\addplot [color=blue,mark size=1.8pt,only marks,mark=o,mark options={solid,draw=mycolor3},forget plot]
  table[row sep=crcr]{figures/2nd-order-continuation_noticks-22.tsv};
\addplot [color=blue,mark size=1.8pt,only marks,mark=*,mark options={solid,fill=mycolor3,draw=mycolor3},forget plot]
  table[row sep=crcr]{figures/2nd-order-continuation_noticks-23.tsv};
\addplot [color=black,solid,line width=1.0pt,forget plot]
  table[row sep=crcr]{figures/2nd-order-continuation_noticks-24.tsv};
\end{axis}
\end{tikzpicture}%
%
\hspace{3mm}%
  }\\\vspace{-1mm}
  \setlength\figureheight{0.12\textheight}
  \mbox{%
  \input{figures/dave/2nd-order-naive-zoom.tikz}%
  \input{figures/dave/2nd-order-smoothing-zoom.tikz}%
  \input{figures/dave/2nd-order-continuation-zoom.tikz}%
  }
  \caption{Options for the continuation of GMRK methods after the
    first extrapolation step (red). All plots use the midpoint method
    and $h=1$. Posterior after two steps (same for all three options)
    in red (mean, $\pm 2$ standard deviations). Extrapolation after 2,
    3, 4 steps (gray vertical lines) in green. Final probabilistic
    prediction as green shading. True solution to (linear) ODE in
    black. Observations of $x$ and $\dot x$ marked by solid and empty
    blue circles, respectively.  Bottom row shows the same data, plotted relative to true solution, at higher
    y-resolution.}
  \label{fig:2nd-order}
\end{figure}

Since Runge-Kutta methods have been extensively studied for over a
century \citep{hairer87:_solvin_ordin_differ_equat_i}, it is not
necessary to evaluate their estimation performance again. Instead, we
focus on an open conceptual question for the further development of
probabilistic Runge-Kutta methods: If we accept high convergence order
as a prerequisite to choose a probabilistic model, how should
probabilistic ODE solvers continue \emph{after} the first $s$ steps?
Purely from an inference perspective, it seems unnatural to introduce
new evaluations of $x$ (as opposed to $\dot x$) at $t_0+nh$ for
$n=1,2,\dots$. Also, with the exception of the Euler case, the
posterior covariance after $s$ evaluations is of such a form that its
renewed use in the next interval will not give Runge-Kutta
estimates. Three options suggest themselves:

\paragraph{Na\"ive Chaining}
\label{sec:naive-chaining}

One could simply re-start the algorithm several times as if the
previous step had created a novel IVP. This amounts to the classic RK
setup. However, it does not produce a joint ``global'' posterior
probability distribution (Figure~\ref{fig:2nd-order}, left column).

\paragraph{Smoothing}
\label{sec:smoothing}

An ad-hoc remedy is to run the algorithm in the ``Na\"ive chaining''
mode above, producing $N\times s$ gradient observations and $N$
function evaluations, but then compute a joint posterior distribution
by using the first $s$ gradient observations and $1$ function
evaluation as described in Section \ref{sec:rel-gp-rk}, then using the
remaining $s(N-1)$ gradients and $(N-1)$ function values as in
standard GP inference.  The appeal of this approach is that it
produces a GP posterior whose mean goes through the RK points
(Figure~\ref{fig:2nd-order}, center column).  But from a probabilistic
standpoint it seems contrived.  In particular, it produces a very
confident posterior covariance, which does not capture global error.

\paragraph{Continuing after $s$ evaluations}
\label{sec:continuing-after-s}

Perhaps most natural from the probabilistic viewpoint is to break with
the RK framework after the first RK step, and simply continue to
collect gradient observations---either at RK locations, or anywhere
else. The strength of this choice is that it produces a continuously
growing marginal variance (Figure~\ref{fig:2nd-order}, right). One may
perceive the departure from the established RK paradigm as
problematic. However, we note again that the core theoretical argument
for RK methods is only strictly valid in the first step, the argument
for iterative continuation is a lot weaker.

Figure~\ref{fig:2nd-order} shows exemplary results for these three
approaches on the (stiff) linear IVP
$\dot{x}(t)=-\nicefrac{1}{2}x(t)$, $x(0)=1$. Na\"ive chaining does not
lead to a globally consistent probability distribution. Smoothing does
give this global distribution, but the ``observations'' of function
values create unnatural nodes of certainty in the posterior. The
probabilistically most appealing mode of continuing inference directly
offers a naturally increasing estimate of global error. At least for
this simple test case, it also happens to work better in practice
(note good match to ground truth in the plots). We have found similar
results for other test cases, notably also for non-stiff linear
differential equations. But of course, probabilistic continuation
breaks with at least the traditional mode of operation for Runge-Kutta
methods, so a closer theoretical evaluation is necessary, which we are
planning for a follow-up publication.

\paragraph{Comparison to Square-Exponential kernel}
\label{sec:comparison}

Since all theoretical guarantees are given in forms of upper bounds
for the RK methods, the application of different GP models might still
be favorable in practice. We compared the continuation method from
Fig.~\ref{fig:2nd-order} (right column) to the ad-hoc choice of a
square-exponential (SE) kernel model, which was used by
\citet{HennigAISTATS2014} (Fig.~\ref{fig:comparison}). For this test
case, the GMRK method surpasses the SE-kernel algorithm both in
accuracy and calibration: its mean is closer to the true solution than
the SE method, and its error bar covers the true solution, while the
SE method is over-confident. This advantage in calibration is likely
due to the more natural choice of the output scale $\ssq$ in the GMRK
framework.

\begin{figure}
  \centering
  \footnotesize
  \setlength\figurewidth{0.9\textwidth}
  \setlength\figureheight{0.12\textheight}%
  \input{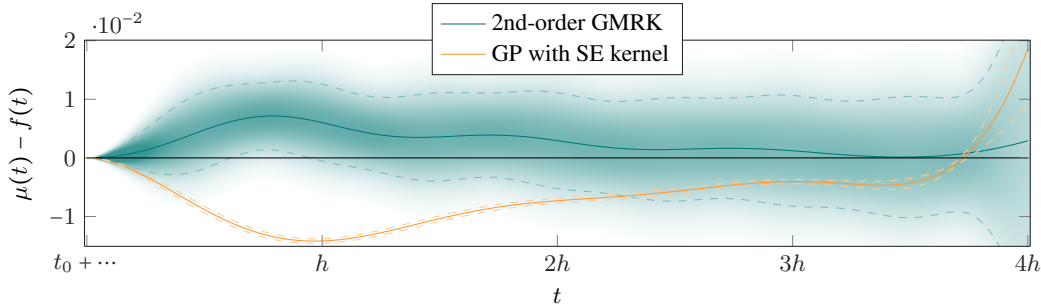}%
\caption{Comparison of a 2nd order GMRK method and the method from 
  \cite{HennigAISTATS2014}. Shown is error and posterior uncertainty
  of GMRK (green) and SE kernel (orange).
  Dashed lines are $+2$ standard deviations.
The SE method shown used the best out of several evaluated parameter choices.
}
  \label{fig:comparison}
\end{figure}


\section{Conclusions}

We derived an interpretation of Runge-Kutta methods in terms of the
limit of Gaussian process regression with integrated Wiener covariance
functions, and a structured but nontrivial extrapolation model. The
result is a class of probabilistic numerical methods returning
Gaussian process posterior distributions whose means can match
Runge-Kutta estimates exactly.

This class of methods has practical value, particularly to machine
learning, where previous work has shown that the probability
distribution returned by GP ODE solvers adds important functionality
over those of point estimators. But these results also raise pressing
open questions about probabilistic ODE solvers. This includes the
question of how the GP interpretation of RK methods can be extended
beyond the 3rd order, and how ODE solvers should proceed after the
first stage of evaluations.


\subsubsection*{Acknowledgments}
The authors are grateful to Simo S\"arkk\"a for a helpful discussion.

\printbibliography[heading=subbibliography]

\clearpage

\begin{center}
{\Large \bf --- Supplementary Material ---}
\end{center}

This document contains derivation steps omitted in the main paper.
Additionally, the website\footnote{\url{http://probabilistic-numerics.org/ODEs.html}}
to this publication contains \matlab{} Symbolic Math Toolbox code which was used
to obtain the lengthy derivations.

\appendix

\section{Multivariate extension}
\label{sec:multi-d-odes}

The GMRK model can be extended to the multivariate case analogously to
Runge-Kutta methods.  All equations in the RK framework also work with
vector-valued function values, and all derivations presented in the
paper and in this supplement carry over to the non-scalar case without
modification: Consider dimension $j \in \{1, \dotsc, N\}$. The
projected outputs are the same as if $j$ were an independent
one-dimensional problem, which can be modeled with a separate Gaussian
process. For a joint notation, vectorize the $N$ dimensions with a
Kronecker product: if $k(t,t')$ is an one-dimensional covariance
function, the function
\begin{equation}
\bar{k}(\vec{t},\vec{t'}) = \mat{D}_{ij} k(t_i,t'_j),
\end{equation}
where $\mat{D}$ is an $N\times N$ positive semi-definite matrix,
defines a covariance function over $N$ dimensions, if $\vec{t}$ and
$\vec{t'}$ are $N$ dimensional. Choosing $\mat{D} = \Id$ results in a
$N$-dimensional GP model where output dimensions are independent of
each other as required.

\section{Covariance functions of integrated Wiener processes}
\label{sec:cov-iwp}
It was observed that integrated Wiener processes generate RK methods of various
order with higher number of integrations leading to RK methods of higher order.

Here we present the derivation of the covariance functions of the applied
Wiener process kernels as well as their derivatives as needed.

\subsection{The once integrated Wiener process}
\label{sec:cov-wp1}
The \emph{Wiener process} covariance function is given by
\begin{equation}
k_{WP}(t,t') = \ssq \min(t,t')\label{eq:wiener-kernel}
\end{equation}
It is only defined for $t,t' > 0$.
Integration with respect to both arguments leads to the
\emph{once integrated Wiener process} which is once differentiable.
Its covariance function is
\begin{align}
k_1(t,t') &= \int_0^t du \int_0^{t'} dv\; \ssq \min(u,v)\notag\\
&= \ssq \int_0^t du \int_0^{t'} dv\; \min(u,v)\notag\\
&\overset{\mathclap{t>t'}}{=} \ssq \left(\int_{t'}^t du \int_0^{t'} dv\; v 
                      + 2 \int_0^{t'} du \int_0^{u} dv\; v\right)\notag\\
&= \ssq \left( \int_{t'}^t du\; \frac{1}{2}t'^2 
          + 2 \int_0^{t'} du\; \frac{1}{2} u^2 \right)\notag\\
&= \ssq \left( \frac{1}{2}(t-t')t'^2 + 
          \frac{1}{3}(t'^3)\right)\notag\\
&= \ssq \left( \frac{\min\nolimits^3(t,t')}{3} +
            \abs{t - t'} \frac{\min\nolimits^2(t,t')}{2}\right)\label{eq:wp1}
\end{align}
where $t,t'$ were replaced with $\min(t,t')$ and $\max(t,t')$ at the last step.

The necessary derivatives of this kernel are
\begin{align}
\kd(t,t') &= \ssq \begin{cases}  t < t': &\frac{t^2}{2} \\
             t > t': &(tt' - \frac{t'^2}{2}) \end{cases} \\
\dkd(t,t') &= \ssq \min(t,t')\qqqq = k_{WP}(t,t').
\end{align}

\subsection{The twice integrated Wiener process}
\label{sec:cov-wp2}
Iterating this process leads to the \emph{twice integrated Wiener
  process}. Its covariance function is
\begin{align}
k_2(t,t') &= \ssq \left( \int_0^t du \int_0^{t'} dv\;
  \frac{\min\nolimits^3(u,v)}{3} +
  \abs{u-v}\frac{\min\nolimits^2(u,v)}{2}\right)\notag\\
  &\begin{aligned}\overset{\mathclap{t>t'}}{=} \ssq \left(
     \left( \int_{t'}^t du \int_0^{t'} dv\; (u-v)
     \frac{v^2}{2} + \frac{v^3}{3} \right)\right. &+ 
     \left( \int_0^{t'} du \int_0^u dv\; (u-v) \frac{v^2}{2} +
     \frac{v^3}{3} \right) \\   &+\left.
     \left( \int_0^v du \int_0^{t'} dv\; (v-u) \frac{u^2}{2} +
     \frac{u^3}{3} \right)\right)
    \end{aligned}
  \notag\\
  &\vdots\notag\\
  &= \ssq \left(\frac{\min\nolimits^5(t,t')}{20} 
            +\frac{\abs{t-t'}}{12} \left((t+t')\min\nolimits^3(t,t')
                                        - \frac{\min\nolimits^4(t,t')}{2}
                                   \right)\right)\label{eq:wp2}
\end{align}

Derivatives of this kernel are
\begin{align}
\kd(t,t') &= \ssq
  \begin{cases}
    t > t': &\left(\frac{{t'}^2}{24}({t'}^2- 4 tt' + 6t^2)\right) \\
    t \leq t': &\left(-\frac{t^4}{24} + \frac{t't^3}{6}\right)
  \end{cases}\\
\dkd (t,t') &= \ssq \left(\frac{\min\nolimits^3(t,t')}{3} +
  \abs{t-t'}\frac{\min\nolimits^2(t,t')}{2}\right) \qq = k_1(t,t').
\end{align}

\subsection{The thrice integrated Wiener process}
\label{sec:cov-wp3}
Similarly, the \emph{thrice integrated Wiener process} is obtained by
\begin{align}
k_3(t,t') &= \ssq \left( \int_0^t du \int_0^{t'} dv\;
    \frac{\min\nolimits^5(u,v)}{20} 
         +\frac{\abs{u-v}}{12} \left((u+v)\min\nolimits^3(u,v)
                                     - \frac{\min\nolimits^4(u,v)}{2}
                                   \right)\right)\notag\\
  &\vdots\notag\\
  &= \ssq \left(\frac{\min\nolimits^7(t,t')}{252}
           + \frac{\abs{t-t'}\min\nolimits^4(t,t')}{720}
            \left( 5\max\nolimits^2(t,t') + 2tt' + 3\min\nolimits^2(t,t')\right)
       \right)\label{eq:wp3}
\end{align}
Omitted steps are similar as in the derivation of \eqref{eq:wp1}
and \eqref{eq:wp2}.

Its derivatives are given by
\begin{align}
\kd(t,t') &= \ssq
  \begin{cases}
    t > t': &\left(\frac{t'^3}{720}\left(20t^3 - 15t^2t' + 6tt'^2 - t'^3\right)\right) \\
    t \leq t': &\left(\frac{t^4}{720}\left(15t'^2 -6tt' + t^2\right)\right)
  \end{cases} \\
\dkd (t,t') &= \ssq\left( \frac{\min\nolimits^5(t,t')}{20} 
            +\frac{\abs{t-t'}}{12} \left((t+t')\min\nolimits^3(t,t')
                                        - \frac{\min\nolimits^4(t,t')}{2}
                                   \right)\right)\q = k_2(t,t').
\end{align}

\section{Posterior predictive GP distributions}
\label{sec:explicit-posteriors}

In order to build GMRK methods it is necessary to compute closed forms of the
resulting posterior mean and covariance functions after $s$ evaluations. Forms
are given below. In cases where a derivation is omitted, results were obtained
with \matlab{}'s Symbolic Math Toolbox. Code is available online.

\subsection{Posterior predictive mean and covariance functions of the once integrated WP}
\label{sec:WP1-explicit}

Below are the formulas of the posterior mean and covariance of the once
integrated WP after each step.

\begin{align}
\mu_{\vert x_0} (t) &= \frac{k(t,t_0)}{k(t_0,t_0)} x_0\notag\\
  &= \frac{\nicefrac{t_0^3}{3} + \abs{t - t_0}\nicefrac{t_0^2}{2}}
          {\nicefrac{t_0^3}{3}} x_0\notag\\
  &= \left(1 + \abs{t-t_0}\frac{3}{2t_0}\right) x_0\label{eq:wp1-mu-post1}\\
\mu_{\vert x_0, y_1} (t) &= \begin{bmatrix}k(t,t_0) &\kd(t,t_0)\end{bmatrix}
                \underbrace{\begin{bmatrix}k(t_0,t_0) & \kd(t_0,t_0)\\
                                           \dk(t_0,t_0)&\dkd(t_0,t_0)\end{bmatrix}^{-1}}_{\ec K}
                            \begin{pmatrix} x_0\\y_1\end{pmatrix}\notag\\
 &= \begin{cases}
      t \geq t_0: &
        \frac{1}{\abs{K}}\begin{bmatrix}\nicefrac{t^2t_0}{2} - \nicefrac{t^3}{6} &
                                        \nicefrac{t^2}{2}\end{bmatrix}
                         \begin{bmatrix}t_0 & -\nicefrac{t_0^2}{2}\\-\nicefrac{t_0^2}{2} & \nicefrac{t_0^3}{3}\end{bmatrix}
                         \begin{pmatrix} x_0\\y_1\end{pmatrix}\\
      t < t_0: &
        \frac{1}{\abs{K}}\begin{bmatrix}\nicefrac{tt_0^2}{2} - \nicefrac{t_0^3}{6} &
                                        tt_0 - \nicefrac{t_0^2}{2}\end{bmatrix}
                         \begin{bmatrix}t_0 & -\nicefrac{t_0^2}{2}\\-\nicefrac{t_0^2}{2} & \nicefrac{t_0^3}{3}\end{bmatrix}
                         \begin{pmatrix} x_0\\y_1\end{pmatrix}
    \end{cases}\notag\\
  &\vdots\notag\\
 &= \begin{cases}
      t \geq t_0: & x_0 + (t-t_0)y_1\\
      t < t_0: & \frac{3t^2t_0-2t^3}{t_0^3}x_0 -\frac{t^2t_0 - t^3}{t_0^2}y_1
    \end{cases}\label{eq:wp1-mu-post2}
\end{align}
Eqs.~\eqref{eq:wp1-mu-post1} and \eqref{eq:wp1-mu-post2} also complete the proof
of Theorem 1 by observing that $y_1 = f(x_0, t_0) = f(\mu_{\vert x_0}(t_0), t_0)$
and $x_1 = x_0 + h y_1 = \mu_{\vert x_0, y_1} (t_0 + h)$, which match Euler's method.

Without loss of generality, we can assume that $t' \leq t$. With this convention
the posterior covariances functions are
\begin{align}
k_{\vert x_0} (t,t') &= k^1(t,t') - \frac{k^1(t,t_0)k^1(t_0,t')}{k^1(t_0,t_0)}\notag\\
  &\begin{aligned}=
    &\left( \frac{\min\nolimits^3(t,t')}{3} +
            \abs{t - t'} \frac{\min\nolimits^2(t,t')}{2}\right)\\
    &- \frac{1}{24t_0^3}\left(\min\nolimits^2(t,t_0)\min\nolimits^2(t_0,t')(t+t_0+2\abs{t-t_0})(t'+t_0+2\abs{t'-t_0})\right)
    \end{aligned}\\
k_{\vert x_0, y_1} (t,t') &= k^1(t,t') 
                           - \begin{bmatrix}k(t,t_0) &\kd(t,t_0)\end{bmatrix} K^{-1}
                             \begin{bmatrix}k(t_0,t') \\ \dk(t_0,t')\end{bmatrix}\notag\\
  &\vdots\notag\\
  &= \begin{cases}
     t,t' > t_0: & \frac{(t_0 - t')^2(3t-t'-2t_0)}{6}\\
     t > t_0 \geq t': &0\\
     t,t' \leq t_0: &\frac{t'^2(t_0 - t)^2(3tt_0-t't_0-2tt')}{6t_0^3}\\
     \end{cases}
\end{align}

\subsection{Predictive mean and covariance of the twice integrated WP}
\label{sec:WP2-explicit}

Below are the formulas for the posterior mean for the twice integrated WP and
the generic 2-stage RK method. Throughout, we write $\mu(t) = \mu(t_0 + s)$ with
appropriate $s \in \Re$, which will simplify formulas significantly.
Furthermore, we omit stating the generating formulas and intermediate steps as
the former are analogous to the ones in Sec.~\ref{sec:WP1-explicit} and the
latter were performed with \matlab{}'s Symbolic Math Toolbox.
\begin{align}
\mu_{\vert x_0} (t_0 + s) &=
  \left[1 + \frac{5s}{2t_0} + \frac{5s^2}{3t_0^2}
                  + \ind{(-\infty,s)}\left(\frac{s^5}{6t_0^5}\right)\right] x_0 \\
\mu_{\vert x_0, y_1} (t_0 + s) &=\begin{aligned}
  &\left[1 - \frac{10s^2}{3t_0^2} 
           + \ind{(-\infty,s)}\left(\frac{5s^4}{t_0^4} + \frac{8s^5}{3t_0^5}\right)\right] x_0\\
+ &\left[s + \frac{2s^2}{t_0}
        - \ind{(-\infty,s)}\left(\frac{2s^4}{t_0^3} + \frac{s^5}{t_0^4}\right)\right] y_0
  \end{aligned}\\
\lim_{t_0 \to \infty} \mu_{\vert x_0, y_1, y_2} (t_0 + s) &=
  x_0 + \left(s - \frac{s^2}{2h\alpha}\right)y_1 + \frac{s^2}{2h\alpha}y_2
\end{align}

As was the case for the posterior mean, we will write the posterior covariance
function as $k(t,t') = k(t_0 + s, t_0 + s')$ while also assuming w.l.o.g.
that $s' \leq s$. The posterior covariance functions are then given by:
\begin{align}
k_{\vert x_0} (t_0 + s, t_0 + s')= 
&\frac{ss'}{48}t_0^3 + \left(ss'^2 + s^2s'\right)\frac{t_0^2}{24}+ \frac{s^2s'^2}{9}t_0\notag\\
+ &\left(\frac{\abs{s}^5}{240} + \frac{s^2s'^3}{24} + \frac{s^3s'^2}{24}
    - \frac{ss'^4}{48} - \frac{s^4s'}{48} + \frac{\abs{s'}^5}{240} - \frac{\abs{s' - s}^5}{240}\right)t_0^0\notag\\
- &\left(ss'^5 + s^5s' - s\abs{s'}^5 - s'\abs{s}^5\right)\frac{t_0^{-1}}{96}
-  \left(s^2s'^5 + s^5s'^2 - s^2\abs{s'}^5 - s'^2\abs{s}^5\right)\frac{t_0^{-2}}{144}\notag\\
- &\left(s^5 - \abs{s}^5\right)\left(s'^5 - \abs{s'}^5\right)\frac{t_0^{-5}}{2880}\\
&\xrightarrow[t_0 \to \infty]{}\infty\notag
\end{align}
\begin{align}
k_{\vert x_0, y_1} (t_0 + s,t_0 + s') = &
\begin{cases}
s, s' > 0: & \frac{s'^5}{240} + \frac{s^2s'^3}{24} + \frac{s^3s'^2}{24} - \frac{\abs{s' - s}^5}{240} + \frac{s^5}{240} - \frac{ss'^4}{48} - \frac{s^4s'}{48} + \frac{s^2s'^2}{36}t_0\\
s > 0 \geq s': &\frac{s^2s'^2(s' + t_{0})^3}{36{t_{0}}^2}\\
s, s' \leq 0: &\frac{s^2s'^3}{24} - \frac{s'^5}{240} + \frac{s^3s'^2}{24} 
- \frac{\abs{s' - s}^5}{240} 
- \frac{s^5}{240} + \frac{ss'^4}{48} + \frac{s^4s'}{48}\\
&+ \frac{s^2s'^2}{36}t_0 + \frac{s^2s'^2(s^2 + s'^2)}{12t_0} + \frac{s^2s'^2(s^3 + s'^3)}{36t_0^2}
- \frac{s^4s'^4(s' + s)}{24t_0^4}\\
&- \frac{s^4s'^4}{12{t_{0}}^3} - \frac{s^5s'^5}{45{t_{0}}^5}
\end{cases}\\
&\xrightarrow[t_0 \to \infty]{}\infty\notag
\end{align}

\newcommand{\ha}{(h\alpha)}

For the final posterior covariance, it is also necessary to distinguish between
the cases whether $s,s' \geq h\alpha$, $s \geq h\alpha > s'$
and $s, s' \leq h\alpha$.
\begin{align}
k_{\vert x_0, y_1, y_2} (t_0 + s,t_0 + s') = &
\begin{cases}
s, s' > \ha: &
\left[\left(8s'^5- 40ss'^4 + 80(s^2s'^3 + \ha^2(s^2s' + ss'^2))\right.\right.\\
&\left.\left. - 20\ha^3(s^2 +s'^2) - 160\ha s^2s'^2\right)t_0\right.\\
&- 15\ha^6 + 60\ha^5(s + s') - 90\ha^4(s^2 + s'^2)\\
& + 24\ha s'^5 + 360\ha^3(ss'^2 + s^2s')\\
& - 120\ha ss'^4 - 540\ha^2s^2s'^2 + 240\ha s^2s'^3\\
&\left. - 240\ha^4ss'\right]\left[960t_{0} + 2880h\alpha\right]^{-1}\\
s > \ha \geq s' > 0: &\left[\left( 20(s^2s'^4 - (h\alpha)^4s'^2) + 80(h\alpha)^3ss'^2+ 8(h\alpha)s'^5\right.\right.\\
&\left.-40((h\alpha)^2s^2s'^2 + (h\alpha)ss'^4)\right)t_0+24(h\alpha)^2s'^5 \\
&+ 15(h\alpha)^3s'^4 - 60(h\alpha)^4s'^3- 180(h\alpha)^2ss'^4\\
&\left.+ 240(h\alpha)^3ss'^3 - 120(h\alpha)^2s^2s'^3 + 90(h\alpha)s^2s'^4\right]\\
& \left[960t_0h\alpha + 2880(h\alpha)^2\right]^{-1}\\
s > \ha > 0 \geq s': &\left[\ha s'^2 (s'+t_0)^3\left(4\ha s -2s^2 - \ha^2\right)\right]\\
&\left[48t_0^2(t_0 + 3\ha) \right]^{-1}\\
\ha \geq s,s' > 0: &\left[\left(80\ha^2s^2s'^2(\ha-s) - 40\ha^2ss'^4\right.\right.\\
&\left.+8\ha^2s'^5 + 20\ha s^2s'^2(s^2+s'^2)\right)t_0\\
&- 15s^4s'^4 + 24\ha^3s'^5 - 120\ha^3ss'^4\\
&+240\ha^2s^2s'^3(\ha-s)\\
&\left. + 60\ha s^3s'^3(s + s') \right]\\
&\left[960t_0\ha^2 + 2880\ha^3 \right]^{-1}\\
\ha \geq s > 0 \geq s': &\left[s^2s'^2(s'+t_0)^3(s-2\ha)^2 \right]\\
&\left[48t_0^2\ha(t_0 + 3\ha) \right]^{-1}\\
s, s' \leq 0: &-\left[s^2(s'+t_0)^3\left(8s^3s'^2 -9s^3s't_0+ 3s^3t_0^2  \right.\right.\\
&\left.\left.+ 15s^2s't_0(s'-t_0)-10s'^2t_0^3\right)\right]\left[360t_0^5\right]^{-1}\\
&-\left[s^2s'^2(s+t_0)^3(s'+t_0)^3\right]\left[36t_0^4(t_0+3\ha)\right]^{-1}\\ 
\end{cases}
\end{align}
\begin{align}
\lim_{t_0 \to \infty}k_{\vert x_0, y_1, y_2} (t_0 + s,t_0 + s') = &
\begin{cases}
s, s' > \ha: &\left[\nicefrac{s'^3}{12} - \nicefrac{\ha s'^2}{6}
+ \nicefrac{\ha^2s'}{12} - \nicefrac{\ha^3}{48}\right]s^2\\
&\left[\nicefrac{\ha^2s'^2}{12} - \nicefrac{s'^4}{24}\right]s
+ \nicefrac{s'^5}{120} - \nicefrac{\ha^3s'^2}{48}\\
s > \ha \geq s' > 0: &\left[s'^2\left(20\ha^3s - 10\ha s(\ha s + s'^2)\right.\right.\\
&\left.\left.+2\ha s'^3 + 5(s^2s'^2 - \ha^4)\right)\right]\\
&\left[240\ha\right]^{-1}\\
s > \ha > 0 \geq s': &-\nicefrac{1}{48}\left[\ha s'^2(\ha^2 - 4\ha s + 2s^2\right]\\
\ha \geq s,s' > 0: &\left[s'^2\left(20\ha s^2(\ha - s) - 10\ha s s'^2\right.\right.\\
&\left.\left.+2\ha s'^3 +5s^2(s^2+s'^2)\right)\right]\left[240\ha\right]^{-1}\\
\ha \geq s > 0 \geq s': &\left[s^2s'^2(s-2\ha)^2\right]\left[48\ha\right]^{-1}\\
s, s' \leq 0: &-\nicefrac{1}{120}\left[s^2(s^3-5s^2s'+10ss'2-10\ha s'^2\right]\\
\end{cases}\label{eq:wp2-cov-post3}
\end{align}
Eq.~\eqref{eq:wp2-cov-post3} concludes the proof of Theorem 2 as it shows that
the covariance function after the RK step is finite for all values of $s, s'$.

\subsection{Posterior predictive mean and covariance of the thrice integrated WP}
\label{sec:WP3-explicit}

Here we list the equations of posterior mean and covariance for the thrice
integrated WP and the generic 3-stage RK method. The same structure as in
Sec.~\ref{sec:WP2-explicit} was applied.

\begin{align}
\lim_{t_0 \to \infty} \mu_{\vert x_0} (t_0 + s) =\;&x_0\\
\lim_{t_0 \to \infty} \mu_{\vert x_0, y_1} (t_0 + s) =\;&x_0 + s y_1\\
\lim_{t_0 \to \infty} \mu_{\vert x_0, y_1, y_2} (t_0 + s) =\; 
&x_0 + \left(s - \frac{s^2}{2hu}\right)y_1 + \frac{s^2}{2hu}y_2 \\
\lim_{t_0 \to \infty} \mu_{\vert x_0, y_1, y_2, y_3} (t_0 + s) =\; 
  &x_0 + \left(s - \frac{h(\frac{s^2u}{2} + \frac{s^2v}{2}) - \frac{s^3}{3}}{h^2uv}\right)y_1\\
  &+ \left(\frac{s^2(2s - 3hv)}{6h^2u(u-v)}\right)y_2 + \left(-\frac{s^2(2s - 3hu)}{6h^2v(u-v)}\right)y_3
\end{align}

As in the case of the twice integrated Wiener process, the covariance function
is infinite for $\lim_{t_0 \to \infty}$. Therefore, we only list the final
posterior covariance function.

For $s, s' > hv > hu > 0$:
\begin{multline}
\lim_{t_0 \to \infty} k_{\vert x_0, y_1, y_2, y_3} (t_0 + s, t_0 + s')=\\
\left\{ \left[-21h^5u(s^2+s'^2) + 14h^4(s^3+s'^3)\right]v^5 + \left[ -21h^5u^2(s^2 + s'^2) + 14h^4u(s^3+s'^3)\right.\right.\\
\left. + 126h^4u(s^2s' + ss'^2) -84h^3(s^3s'+ss'^3)\right]v^4 + \left[ -21h^5u^3(s^2 + s'^2) + 14h^4u^2(s^3+s'^3)\right.\\
\left. +126h^4u^2(s^2s' + ss'^2) -84h^3u(s^3s' + ss'^3) -630h^3us^2s'^2 +210h^2(s^3s'^2 + s^2s'^3)\right]v^3\\
+ \left[ -21h^5u^4(s^2 + s'^2) +14h^4u^3(s^3 + s'^3) + 126h^4u^3(s^2s' + ss'^2) - 84h^3u^2(s^3s' + ss'^3)\right.\\
\left. -252h^3u^2s^2s'^2 + 378h^2u(s^3s'^2 + s^2s'^3) -392hs^3s'^3\right]v^2 + \left[14h^4u^4(s^3 + s'^3)\right.\\
- 84h^3u^3(s^3s' + ss'^3) +126h^2u^2(s^3s'^2 - s^2s'^2 + s^2s'^3) -224hus^3s'^3 + 210s^3s'^4 -126s^2s'^5\\
\left. + 42ss'^6 - 6s'^7\right]v + 42h^2u^3(s^3s'^2 + s^2s'^3)\left. -56hu^2s^3s'^3\right\}/(30240v)\\
\end{multline}
For $s > hv \geq s' > hu > 0$:
\begin{multline}
\lim_{t_0 \to \infty} k_{\vert x_0, y_1, y_2, y_3} (t_0 + s, t_0 + s')=\\
\big\{ \left(21h^7us'^2 - 14h^6s'^3\right)v^6 + \left( -126suh^6s'^2 +
      84sh^5s'^3\right)v^5\\
    + (315 u h^5 s^2 s'^2 - 210 h^4 s^2 s'^3)v^4\\
    + (-378 h^5 s^2 s'^2
    u^2 - 168 h^4 s^3 s'^2 u + 252 h^4 s^2 s'^3 u + 112 h^3 s^3 s'^3)
    v^3\\ + 
    (-21 h^7 s^2 u^5 - 21 h^7 s'^2 u^5 + 126 h^6 s^2 s' u^4 + 126 h^6
    s s'^2 u^4 - 126 h^5 s^2 s'^2 u^3 + 252 h^4 s^3 s'^2 u^2\\
    + 252 h^4
    s^2 s'^3 u^2 - 168 h^3 s^3 s'^3 u - 315 h^3 s^2 s'^4 u + 126 h^2
    s^2 s'^5 - 42 h^2 s s'^6 + 6 h^2 s'^7)v^2\\
    + (14h^6 s^3 u^5 - 84 h^5 s^3 s' u^4 - 126 h^5 s^2 s'^2 u^4 -
    84 h^5 s s'^3 u^4 + 84 h^4 s^3 s'^2 u^3\\ + 84 h^4 s^2 s'^3 u^3 -
    168 h^3 s^3 s'^3 u^2 + 210 h^2 s^3 s'^4 u + 42 h^2 s s'^6 u - 6
    h^2 s'^7 u - 84 h s^3 s'^5 ) v \\
+42 h^4 s^3 s'^2 u^4 + 42 h^4 s^2 s'^3 u^4 - 56 h^3 s^3 s'^3 u^3 - 21
h s^2 s'^6 u + 14 s^3 s'^6\big\}/(-30240 h^2 v^2 + 30240 uh^2 v)
\end{multline}
For $s > hv > hu \geq s' > 0$:
\begin{multline}
\lim_{t_0 \to \infty} k_{\vert x_0, y_1, y_2, y_3} (t_0 + s, t_0 +
s')=\\
\frac{s'^2}{30240h^2uv}
\big[ -21h^7(u^5v^2 + u^4v^3 + u^3v^4 +u^2v^5) - 21hs^2s'^4(u+v)-126h^4s^2u^3v(hu - s')\\
+126h^6s(u^4v^2+u^3v^3+u^2v^4) +14h^6s'(u^5v + u^4v^2 +u^3v^3 + u^2v^4 + uv^5) + 14s^3s'^4\\
 + 63h^5u^3v^2s^2-315h^5u^2v^3s^2 -84h^5ss'(u^4v+u^3v^2+u^2v^3+uv^4) -84h^4u^3vs^3\\
 +42h^4u^4s^2(s + s')-42h^4u^2v^2s^2s' + 42h^2uvss'^4 +168h^4u^2v^2s^3 +210h^4uv^3s^2s'\\
 -56h^3u^2s^3s'(u-v)-112h^3uv^2s^3s' -6h^2uvs'^5\big]
\end{multline}
For $s > hv > hu > 0 \geq s'$:
\begin{multline}
  \lim_{t_0 \to \infty} k_{\vert x_0, y_1, y_2, y_3} (t_0 + s, t_0 +
  s')=\\
\frac{h s'^2}{4320 v}
  \big[  (-3 h^4 u^4 v^2 - 3h^4 u^3 v^3 - 3h^4 u^2 v^4 -
  3h^4uv^5
  + 18 h^3 su^3 v^2 + 18 h^3 s u^2 v^3 + 18 h^3 suv^4\\
  + 2s' h^3 u^4 v + 2s' h^3 u^3 v^2 + 2 s' h^3 u^2 v^3 + 2 s' h^3 u
  v^4 + 2 s' h^3 v^5 - 18 h^2 s^2 u^3 v + 9 h^2 s^2 u^2 v^2\\
  - 45 h^2 s^2 u v^3 - 12 s' h^2 s u^3 v - 12 s' h^2 s u^2 v^2 - 12 s'
  h^2 suv^3 - 12 s'h^2 sv^4 + 6 hs^3 u^3\\
  - 12hs^3 u^2 v + 24 hs^3 uv^2 + 6s' hs^2 u^3 + 18 s' h s^2 u^2 v 
  - 6 s' h s^2 u v^2  + 30 s' h s^2 v^3 - 8 s' s^3 u^2 \\
  + 8 s' s^3 u v - 16 s' s^3 v^2 )\big]
\end{multline}
For $hv \geq s, s' > hu > 0$:
\begin{multline}
\lim_{t_0 \to \infty} k_{\vert x_0, y_1, y_2, y_3} (t_0 + s, t_0 +
s')=\\
\big\{(378 h^5 s^2 s'^2 u^2 - 252 h^4 s^3 s'^2 u - 252 h^4 s^2 s'^3 u
+ 168 h^3 s^3 s'^3 )v^3\\
+ (21 h^7 s^2 u^5 + 21 h^7 s'^2 u^5 - 126 h^6 s^2 s' u^4  - 126 h^6 s
s'^2 u^4 + 126 h^5 s^2 s'^2 u^3 \\
- 252 h^4 s^3 s'^2 u^2 - 252 h^4 s^2 s'^3 u^2 + 315 h^3 s^4 s'^2 u + 
168 h^3 s^3 s'^3 u + 315 h^3 s^2 s'^4 u'\\
-210 h^2 s^4 s'^3 - 126 h^2 s^2 s'^5 + 42 h^2 s s'^6 - 6h^2 s'^7)v^2\\
+ (-14 h^6 s^3 u^5 - 14h^6 s'^3 u^5 + 84 h^5 s^3 s' u^4 + 126 h^5 s^2
s'^2 u^4 + 84 h^5 s s'^3 u^4 \\
- 84 h^4 s^3 s'^2 u^3 - 84 h^4 s^2 s'^3 u^3 + 168 h^3 s^3 s'^3 u^2 
-126 h^2 s^5 s'^2 u\\
- 126 h^2 s^5 s'^2 u - 210 h^2 s^3 s'^4 u - 42 h^2 s s'^6 u + 6 h^2
s'^7 u + 84 h s^5 s'^3 + 84 h s^3 s'^6 )v\\
-42 h^4 s^3 s'^2 u^4 - 42 h^4 s^2 s'^3 u^4 + 56 h^3 s^3 s'^3 u^3 + 21
h s^6 s'^2 u + 21 h s^2 s'^6 u\\
- 14 s^6 s'^3 - 14 s^3 s'^6\big\}/(30240 h^2 v^2 - 30240 uh^2 v)
\end{multline}
For $hv \geq s > hu \geq s' > 0$:
\begin{multline}
\lim_{t_0 \to \infty} k_{\vert x_0, y_1, y_2, y_3} (t_0 + s, t_0 + s')=\\
\frac{s'^2}{30240h^2uv(u-v)}
\big[ -21h^7u^6v^2 -21hu^2s^6 -21hs^2s'^4(u^2-v^2) +126h^2u^2vs^5\\
+126h^4u^4vs(h^2uv-hus-s^2) +14s'(h^6u^6v + s^6u) +14s^3s'^4(u-v) +189h^5u^4v^2s^2 \\
-378h^5u^3v^3s^2 -84h^4u^4vss'(hu - s) -84huvs^5s' +42h^4u^5s^2(s + s')\\
 +42h^2ss'^4(u^2v-uv^2) +252h^4u^2v^2s^2(us+vs+vs') -168h^3uv^2s^2s'(hu^2 + us + vs)\\
 -315h^3u^2v^2s^4-56h^3u^4s^3s' +112h^3u^3vs^3s' +210h^4uv^4s4s' -6h^2s'^5(u^2v-uv^2)\big]\\
\end{multline}
For $hv \geq s > hu > 0 \geq s'$:
\begin{multline}
  \lim_{t_0 \to \infty} k_{\vert x_0, y_1, y_2, y_3} (t_0 + s, t_0 +
  s')=\\
  \frac{s'^2}{4320h^2v(u-v)} \big[ -3h^7u^5v^2 +18h^6su^4v^2
  +2s'h^6u^5v -18h^5s^2u^4v +27h^5s^2u^3v^2 \\-54h^5s^2u^2v^3
  -12s'h^5su^4v +6h^4s^3u^4 -18h^4s^3u^3v +36h^4s^3u^2v^2
  +36h^4s^3uv^3 \\+6s'h^4s^2u^4 +12s'h^4s^2u^3v -24s'h^4s^2u^2v^2
  +36s'h^4s^2uv^3 -45h^3s^4uv^2 -8s'h^3s^3u^3 \\+16s'h^3s^3u^2v
  -24s'h^3s^3uv^2 -24s'h^3s^3v^3 +18h^2s^5uv +30s'h^2s^4s^2 -3hs6u
  \\-12s'h^5v +2s's^6 \big]
\end{multline}
For $hv > hu \geq s, s' > 0$:
\begin{multline}
  \lim_{t_0 \to \infty} k_{\vert x_0, y_1, y_2, y_3} (t_0 + s, t_0 + s')=\\
  - \frac{s'^2}{30240h^2uv} \big[ 126h^5s^2u^4v - 378h^5s^2u^3v^2 -
  42h^4s^3u^4 + 84h^4s^3u^3v +252h^4s^3u^2v^2 \\- 42h^4s^2s'u^4 +
  84h^4s^2s'u^3v + 252h^4s^2s'u^2v^2 + 56h^3s^3s'u^3 -336h^3s^3s'u^2v
  \\ -168h^3s^3s'uv^2 -126h^2s^5uv + 210h^2s^4s'uv - 42h^2ss'^4uv +
  6h^2s'^5uv \\+21hs^6v +21hs^2s'^4u + 21hs^2s'^4v -14s^6s' -14s^3s'^4
  \big]
\end{multline}
For $hv > hu \geq s > 0 \geq s'$:
\begin{multline}
  \lim_{t_0 \to \infty} k_{\vert x_0, y_1, y_2, y_3} (t_0 + s, t_0 +
  s')=\\
  - \frac{s^2 s'^2}{4320 h^2 u v} \big[ 18 h^5 u^4 v - 54 h^5 u^3 v^2 - 6 h^4 s u^4 \\
  + 12h^4 s u^3v + 36 h^4 s u^2v^2 - 6 s' h^4 u^4 + 12 s' h^4 u^3 v +
  36 s' h^4 u^2 v^2 + 8 s' h^3 s u^3 \\- 48 s' h^3 s u^2 v - 24 s' h^3
  s u v^2 - 18 h^2 s^3 u v + 30 s' h^2 s^2 u uv + 3 h s^4 v - 2s' s^4
  \big]
\end{multline}
For $hv > hu > 0 \geq s, s'$:
\begin{multline}
  \lim_{t_0 \to \infty} k_{\vert x_0, y_1, y_2, y_3} (t_0 + s, t_0 +
  s')=\\
  \frac{hs^2 s'^2 u^2 \left( h s u - 4 s s' + 3 h s' u \right)}{2160 v} \\
  - \frac{ s^2}{5040} \big[ 21 h^3 s'^2 u^3 - 63 v h^3 s'^2 u^2 + 14h^2 s s'^2 u^2 + 42 v h^2 s s'^2  \\
  + 14 h^2 s'^3 u^2 + 42 v h^2 s'^3 u - 56 h s s'^3 u - 28 v h s s'^3  \\
  - s^5 + 7s^4 s' - 21 s^3 s'^2 + 35 s^2 + s'^3 \big] \\
\end{multline}


\section{Square-exponential kernel cannot yield Euler's method}
\label{sec:se-kernel}

We show that the square-exponetial (SE, aka. RBF, Gaussian) kernel
cannot yield Euler's method for finite length-scales.

\newcommand{\tsq}{\theta^2}
\newcommand{\lsq}{\lambda^2}

The SE kernel and its derivatives are
\begin{align}
k(t,t') &= \tsq \exp(\nicefrac{-\norm{t-t'}^2}{2\lsq})\\
\kd(t,t') &= \frac{(t-t')}{\lsq}k(t,t')\\
\dkd(t,t') &= \left[\frac{1}{\lsq} 
            - \left(\frac{(t-t')}{\lsq}\right)^2\right]k(t,t')
\end{align}

To show that this choice does not yield Euler's method, we proceed as in the
case for the GMRK methods. The predictive mean after observing $x_0$ and $y_1$
is given by
\begin{align}
\mu_{\vert x_0, y_1} (t_0 + s) &=
  \begin{bmatrix}k(t_0+s,t_0) &\kd(t_0+s,t_0)\end{bmatrix}
  \underbrace{\begin{bmatrix}k(t_0,t_0) & \kd(t_0,t_0)\\
                             \dk(t_0,t_0)&\dkd(t_0,t_0)\end{bmatrix}^{-1}}_{\ec K}
  \begin{pmatrix} x_0\\y_1\end{pmatrix}\notag\\
&= \begin{bmatrix} \tsq \exp(\nicefrac{-\norm{s}^2}{2\lsq}) &
                   s \frac{\tsq \exp(\nicefrac{-\norm{s}^2}{2\lsq})}{\lsq}
   \end{bmatrix}
  \begin{bmatrix} \tsq & 0\\ 0 & \nicefrac{\tsq}{\lsq}\end{bmatrix}^{-1}
  \begin{pmatrix} x_0\\y_1\end{pmatrix}\notag\\
&=\begin{bmatrix} \exp(\nicefrac{-\norm{s}^2}{2\lsq}) &
                s \exp(\nicefrac{-\norm{s}^2}{2\lsq})\end{bmatrix}
 \begin{pmatrix} x_0\\y_1\end{pmatrix}\notag\\
&=\exp(\nicefrac{-\norm{s}^2}{2\lsq})x_0 + s \exp(\nicefrac{-\norm{s}^2}{2\lsq})y_1\label{eq:se-posterior}
\intertext{evaluated at $h$ yields}
&=\exp(\nicefrac{-\norm{h}^2}{2\lsq})x_0 + h \exp(\nicefrac{-\norm{h}^2}{2\lsq})y_1\label{eq:se-post-at-h}
\end{align}
An interesting observation, left out in the main paper to avoid
confusion, is that Eq.~\eqref{eq:se-post-at-h} \emph{does} indeed
produce the weights for Euler's method for the limit $\lim_{\lambda
  \to \infty}$. In fact, it can even be used to derive second and
third order Runge-Kutta means, too. Future work will provide more
insight into this property. However, this limit in the length-scale
yields a Gaussian process posterior that has little use as a
probabilistic numerical method, because its posterior covariance
vanishes everywhere. This is in contrast to the integrated Wiener
processes discussed in the paper, which yield proper finite,
interpretable posterior variances, even after in the limit in
$\tau$. Finally, SE kernel-GPs are not Markov. Inference in these
models has cost cubic in the number of observations, reducing their
utility as numerical methods.

\end{document}